\documentclass[twoside]{article} 
\usepackage[accepted]{aistats2017}
\usepackage{graphicx}
\graphicspath{ {images/} }
\usepackage{subfigure}
\usepackage{paralist}
\usepackage{booktabs} 

\usepackage{cite}
\usepackage{algorithm}
\usepackage{algorithmic}
\usepackage{amsmath,amssymb,xspace,amsthm,Tabbing,bm}
\usepackage{color,mathtools}
\usepackage{amsthm}
\usepackage{listings}
\usepackage{multirow}
\usepackage{fullpage}

\usepackage{wrapfig}
\usepackage{tikz}
\usepackage{wasysym}

\newcommand{\punt}[1]{}

\def\argmin{\mathop{\rm arg\,min}}

\newcommand{\reals}{\mathbb{R}}

\def\argmin{\mathop{\rm arg\,min}}

\newcommand{\bq}{\begin{equation}}
\newcommand{\eq}{\end{equation}}
\newcommand{\ba}{\begin{eqnarray}}
\newcommand{\ea}{\end{eqnarray}}

\def\la{{\langle}}
\def\ra{{\rangle}}

\newcommand{\mbb}[1]{\mathbb{#1}}
\newcommand{\mcal}[1]{\mathcal{#1}}
\newcommand{\remove}[1]{}

\DeclareMathOperator{\diam}{diam}

\newcommand{\proposed}{LGIKA}

\newtheorem{theorem}{Theorem}[section]

\newtheorem{lemma}[theorem]{Lemma}

\allowdisplaybreaks[4]

\begin{document}

%

%

\twocolumn[

\aistatstitle{Local Group Invariant Representations via Orbit Embeddings}

\aistatsauthor{\hspace{1mm} Anant Raj \hspace{1mm}  Abhishek Kumar \hspace{1mm} Youssef Mroueh \hspace{1mm} P. Thomas Fletcher \hspace{1mm} Bernhard Sch\"olkopf}

\aistatsaddress{\hspace{-10mm} MPI \hspace{7mm} AIF, IBM Research \hspace{2mm}  AIF, IBM Research \hspace{3mm}  University of Utah \hspace{13mm} MPI } ]


\begin{abstract}
Invariance to nuisance transformations is one of the desirable properties
of effective representations. We consider transformations that form a \emph{group}
and propose an approach based on kernel methods to derive local group invariant representations. 
Locality is achieved by defining a suitable probability distribution over 
the group which in turn induces distributions in the input feature space.  
We learn a decision function over these distributions by appealing to the powerful framework of kernel methods
and generate local invariant random feature maps via kernel approximations.
We show uniform convergence bounds for kernel approximation and provide generalization bounds for learning
with these features. 
We evaluate our method on three real datasets, including Rotated MNIST and CIFAR-10, and 
observe that it outperforms competing kernel based approaches. 
The proposed method also outperforms deep CNN on Rotated-MNIST and performs 
comparably to the recently proposed group-equivariant CNN.

\end{abstract}
 
\vspace{-7mm}
\section{Introduction}
\label{sec:intro}
\vspace{-2mm}
Effective representation of data plays a key role in the success of learning algorithms. 
One of the most desirable properties of effective representations is being invariant to nuisance transformations.
For instance, convolutional neural networks (CNNs) owe much of their empirical success to their ability in 
capturing local translation invariance through convolutional weight sharing and pooling which turns out to be a useful model prior for images. 
Capturing class sensitive invariance can also result in reduction in sample complexity~\cite{poggio2014inv} 
which is particularly useful in label scarce applications. 
We approach the problem of learning with invariant representations from a group theoretical perspective and 
propose a scalable framework for incorporating invariance to nuisance group actions via kernel methods. 

At an abstract level, a \emph{group} is defined as a set $G$ endowed with a notion of \emph{product} on its elements that
satisfies certain axioms of (i) \emph{closure}: $a,b\in G \Longrightarrow$ the product $ab\in G$, 
(ii) \emph{associativity}: $(ab)c = a(bc)$, and (iii) \emph{inverse element}: for each $g\in G, \exists g^{-1}\in G$
such that $gg^{-1} = g^{-1}g = e \in G$, where $e$ is the identity element satisfying $ge=eg=g, \forall\,g\in G$.
A group is \emph{abelian} if the group product is commutative ($gh = hg, \forall\, g,h\in G$).   
For most practical applications each element $g\in G$ can be seen as a transformation acting on an input space $X$, $T_g: X\mapsto X$. 
The \emph{orbit} of an element $x\in X$ under the action of the group $G$ is defined as the set $O_x = \{T_g(x) \mid g\in G\}$. 
The set of all rotations in a fixed 2-D plane is an example of an infinite group where the product is defined as the consecutive application
of two rotations. The orbit of an image under this rotation group is the infinite set consisting of all rotated versions of the image. 
The closure property of the group implies that the orbit of a point $x$ is invariant under a group action on $x$, i.e., 
$O_{x\mid G} = O_{T_g(x)}, \forall g\in G$. 
The reader is referred to \cite{introgrouptheory} for a more detailed introduction to group theory. 

For unimodular groups, which include compact groups and abelian groups,
there exists a so called unique (up to scaling) \emph{Haar measure} $\nu$ that is invariant to 
both left and right group products, 
i.e., $\nu(S)=\nu(gS)=\nu(Sg)$
for all measurable subsets $S\subset G$ and all $g\in G$, essentially generalizing the notion of Lebesgue measure to groups. 
For a compact group $G$, Haar measure can be normalized by $\nu(G)$ (since $\nu(G)<\infty$) to obtain the normalized Haar measure
which assigns a probability mass to all measurable subsets of $G$. Normalized Haar measure can be seen
as inducing a uniform probability distribution on the group. Recently, Anselmi et al. \cite{poggio2014inv} used the normalized Haar measure
$\tilde{\nu}$ on the group to map each orbit ($O_{x\mid G} \forall x$) to a probability distribution $P_x$ on the input space, i.e., 
$P_{x\mid G}(A)=\tilde{\nu}(\{g \mid T_g(x)\in A\}), \forall A\subset X$. The distribution $P_{x\mid G}$ induced by each point $x$ can be taken
as its invariant representation. However, estimating this distribution directly can be challenging due to its potentially high dimensional support. 
Anselmi et al. \cite{poggio2014inv} propose to capture histogram statistics of 1-dimensional projections of $P_{x\mid G}$
to generate an invariant representation that can be used for learning, i.e., 
$\phi_n^k(x)=1/|G| \sum_{g\in G} \eta_n(\langle T_g(x), t_k\rangle)$ for a finite group $G$, where $t_k$ are the projection directions 
(termed as \emph{templates}), $\eta_n(\cdot)$ are some nonlinear functions that are expected to capture the histogram statistics. More recently, 
Mroueh et al. \cite{mroueh2015} analyzed the concentration properties of the linear kernel defined over these features
and provided generalization bounds for learning with this linear kernel. 

Our point of departure from \cite{poggio2014inv,mroueh2015} is the observation that histogram based features
may not be the optimal way to characterize the probability distributions $P_x$ induced by the group on the input space and 
their approach has its limitations. 
First, there is no principled guidance provided regarding the choice of nonlinearities $\eta_n$. 
Second, the inner-product of histogram based features ($\{\phi_n^k(x)\}$) 
approximately 
induces a Euclidean distance (group-averaged) in the input space \cite{mroueh2015} which may render them unsuitable for 
learning complex nonlinear decision boundaries in the input space. Further, \emph{locality} is achieved by
restricting the uniform distribution to a chosen subset of the group (i.e. elements within the subset are allowed to transform the input
with equal probability and elements outside the subset are prohibited) which can be limiting. 

{\bf Contributions:} In this paper, we address aforementioned points and propose a framework to generate invariant representations by embedding the orbit distributions $P_{x\mid G}$
into a reproducing kernel Hilbert space (RKHS)~\cite{kme2007smola,kme2016survey}. 
We propose to use \emph{characteristic} kernels \cite{injective2007bharat} so that the resulting map from the
distributions to the RKHS is injective (one-to-one), preserving all the moments of the distribution. 
Our use of kernel methods to embed orbit distributions also renders a large body of work on kernel approximation methods
at our disposal, which enable us to scale our proposed method. In particular, we derive invariant features by 
approximating the kernel using Nystr\"{o}m method \cite{nystrom2001,nystrom2005} and 
random Fourier features (for shift invariant kernels) \cite{randomfourier07}. 
The nonlinearities in the features ($\eta_n(\cdot)$) emerge in a principled manner as a by-product of the kernel approximation. 
The RKHS embedding framework also naturally allows us to use more general probability distributions on the group, apart from the uniform
distribution. This allows us to have better control over selectivity of the derived features and also 
becomes a technical necessity when the group in non-compact. 
We experiment with three real datasets and observe consistent accuracy improvements over baseline random Fourier \cite{randomfourier07}
and Nystr\"{o}m features \cite{nystrom2005} as well as over \cite{mroueh2015}. 
Further, on Rotated MNIST dataset \cite{larochelle2007empirical} we outperform recent invariant deep CNN and RBM based architectures
\cite{sohnlee12,schmidt12}, and perform comparably to the more recently proposed group equivariant deep convolutional nets \cite{groupequiconvnets}.

\vspace{-3mm}
\section{Formulation}
\label{sec:formulation}
\vspace{-2mm}
Let the input features belong to a set $X\subset \reals^d$. 
A group element $g\in G$ \emph{acts} on points from $\reals^d$ through a map $T_g: \reals^d \mapsto \reals^d$,
and we use a shorthand notation of $gx$ to denote $T_g(x)$. We use $gS$ to denote the action
of a group element $g$ on the set $S$, i.e., $gS = \{T_g(x) \mid x \in S\subseteq X\}$. We take liberty in using the same notation
to denote the product of a group element with a subset of the group, i.e., $gS = \{gh \mid h \in S\subset G\}$
and $Sg = \{hg \mid h \in S\subset G\}$. 

\vspace{-1mm}
\subsection{RKHS embedding of Orbit distributions}
\label{subsec:orbitembed}
\vspace{-1mm}
As introduced in the previous section, the \emph{orbit} of an element $x\in X$ under the action of the group $G$ 
is defined as the set $O_{x\mid G} = \{gx \mid g\in G\}$. 
For all unimodular groups there exists a Haar measure
$\nu: S\mapsto \reals_+$ which is invariant under left and right 
group product 
i.e., $\nu(S)=\nu(gS)=\nu(Sg)$ for all measurable subsets $S\subset G$ and all $g\in G$. 
Let $q(\cdot)$ be the probability density function of a distribution defined over $G$.  
This probability distribution over the group can be used to map each orbit $O_{x\mid G}$ to a 
probability distribution $P_{x\mid G}$ on the input space, i.e., $P_{x\mid G}(A)=\int_{g:gx\in A} q(g)\, d \nu(g)\,\, \forall A\subset X$.
Note that $P_{x\mid G}(O_{x\mid G}) = 1$ (for an appropriately normalized measure $\nu$), 
and $P_{x\mid G}(A)=0\,\forall\, A$ for which $A\cap O_{x\mid G} =\emptyset$. 

Let $\mcal{H}$ be a reproducing kernel Hilbert space (RKHS) of functions $f:X\mapsto \reals$ induced by kernel $k:X\times X \mapsto \reals$, 
with the inner-product satisfying the reproducing property, i.e., $\la f,k(x,\cdot)\ra = f(x), \forall f\in\mcal{H}$ and
$\la k(x,\cdot),k(x',\cdot)\ra = k(x,x')$. 
The RKHS embedding of the distribution $P_{x\mid G}$ is given as~\cite{kme2007smola}
\begin{align}
\mu[P_{x\mid G}] := E_{z\sim P_{x\mid G}} k(z,\cdot). 
\label{eq:ke_def}
\end{align}
This expectation is well-defined under the probability measure $P_{x\mid G}$, which is in turn induced by the measure $\nu$ over the group. The support of $P_{x\mid G}$ is $O_{x\mid G}$ and sampling a point $z\sim P_{x\mid G}$ 
is equivalent to sampling the corresponding group element $g$ and setting $z = gx$. 
Thus we can rewrite the RKHS embedding of Eq.~\ref{eq:ke_def} as
\begin{align}
\mu[P_{x\mid G}] = \int_G k(gx,\cdot) q(g)\, d \nu(g). 
\end{align}
If the kernel is characteristic this map from distributions to the RKHS is injective, preserving all the information
about the distribution \cite{injective2007bharat}. All universal kernels~\cite{steinwart01} are characteristic when the support set
of the distribution is compact~\cite{kme2007smola}. 
In addition, many shift invariant kernels (e.g., Gaussian and Laplacian kernels) are characteristic 
on all of $\reals^d$~\cite{fukumizu2007kernel}.  For precise characterization of characteristic shift invariant kernels, please refer to \cite{sriperumbudur2010hilbert}.

For a characteristic kernel the embedding $\mu[P_{x\mid G}]$ 
can be used as a proxy for $P_{x\mid G}$ in learning problems. To this end, we introduce a hyperkernel
$h:\mcal{H}\times\mcal{H}\mapsto \reals$ that defines the similarity between the RKHS embeddings 
corresponding to two points $x$ and $x'$ as $k_{q,G}(x,x') :=h(\mu[P_{x\mid G}], \mu[P_{x'\mid G}])$. If we take $h$ to be
the linear kernel which is the regular inner-product in $\mcal{H}$, we obtain
\begin{align}
\begin{split}
k_{q,G}(x,x') &:= \la\, \mu[P_{x\mid G}], \mu[P_{x'\mid G}]\, \ra_\mcal{H} \\
&= \int_G \int_G k(gx,g'x') q(g) q(g') d \nu(g) d \nu(g') 
\end{split}
\label{eq:localhaar}
\end{align}
The kernel $k_{q,G}:X\times X\mapsto \reals$ turns out to be the expectation of the \emph{base kernel} $k(\cdot,\cdot)$ 
under the predefined probability distribution on the group $G$. 
It trades off locality and group invariance through appropriately selecting the
probability density $q(\cdot)$. Taking $q$ to be a delta function over the Identity group element gives back
the original base kernel $k(\cdot,\cdot)$ which does not capture any invariance. 
On the other hand, if we take $q$ to be the uniform probability density, 
we get the global group invariant kernel (also termed as \emph{Haar integration kernel}~\cite{haarintkernels,mroueh2015})
\begin{align}
k_G(x,x') = \int_G \int_G k(gx,g'x') d \nu(g) d \nu(g'), 
\label{eq:haar}
\end{align}
\noindent satisfying the property $k_G(gx,g'x') = k_G(x,x')$ for any $g,g'\in G$ and any $x,x'\in X$.
Haar integral kernel does not preserve any locality information (e.g., images of digits $6$ and $9$ will
be placed under same equivalence class). 
Strictly speaking, we only need $\nu$ to be the normalized \emph{right Haar measure} satisfying $\nu(S)=\nu(Sg),\,\forall\,S\subset G,\,\forall\,g\in G$
for the global group invariance property to hold. 
A unique (up to scaling) right Haar measure exists for all locally compact groups and for all
unimodular groups (for which left and right Haar measures conincide)~\cite{introgrouptheory}.
All Lie groups (e.g., rotation, translation, scaling, affine) are locally compact.
Additionally, all compact groups (e.g., rotation), abelian groups (e.g., translation, scaling), and 
discrete groups (e.g., permutation) are unimodular. 
However, the Haar integration kernel $k_G(x,x')$ of Eq.~\ref{eq:haar} can only be defined for compact groups 
since we need $\nu(G)<\infty$ to keep the integral finite. 
Indeed, earlier work has used Haar integration kernel for compact groups~\cite{haarintkernels,mroueh2015} (however, without the RKHS embedding
perspective provided in our work which motivates the use of a  \emph{characteristic} base kernel $k(\cdot,\cdot)$).

A framework allowing more general (non-uniform) probability distribution on the group serves two purposes: 
(i) It enables us to operate
with non-compact groups in a principled manner since we only need $\int_G q(g) d \nu(g) < \infty$ 
to enable construction of kernels such that Eq.~\ref{eq:localhaar} is finite;
(ii) It allows for a better control over locality of the kernel $k_{q,G}(\cdot,\cdot)$. 
Earlier work~\cite{poggio2014inv,mroueh2015} achieves locality by taking a subset $G_0 \subset G$ 
and restricting the domain of the Haar integration kernel to be $G_0$ which amounts to having a uniform distribution over $G_0$. 
A more general non-uniform distribution (e.g., a unimodal distribution with mode at the Identity element of the group)
allows us to smoothly decrease the probability of sampling more extreme group transformations rather than abruptly
prohibiting group transforms falling outside a preselected subset. 

\vspace{1mm}
\subsection{Feature generation via kernel approximation}
\label{subsec:kerapprox}
The kernel $k_{q,G}$ of Eq.~\ref{eq:localhaar} can be used for learning with kernel machines \cite{learningkernelsbook},
probabilistically trading off locality and group invariance through appropriately selecting $q(\cdot)$. 
However, kernel based learning algorithms suffer from scalability issues due to the need to compute
kernel values for all pairs of data points. In this section, we describe our approach to obtain local invariant 
features via approximating $k_{q,G}$.

\subsubsection{Features using random Fourier approximation}
\label{subsubsec:fourier}
We first consider the case of shift-invariant base kernel satisfying $k(x,x') = \tilde{k}(x-x')$ which
is a commonly used class of kernels that includes Gaussian and Laplacian kernels. Many shift-invariant kernels
are characteristic on $\reals^d$ as mentioned in the previous section. 
We use the random Fourier features proposed in \cite{randomfourier07} that are based on the 
characterization of positive definite functions by Bochner \cite{bochner33,rudin2011fourier}. 
Bochner's theorem establishes Fourier transform as a \emph{bijective map} from finite non-negative Borel measures on $\reals^d$
to positive definite functions on $\reals^d$. Applying it to shift-invariant positive definite kernels
one gets 
\begin{align}
k(x,x')=\tilde{k}(x-x') =\int_{\reals^d} e^{-i (x-x')^\top \omega} p(\omega) d \omega,\,\forall\, x,x',
\label{eq:bochner}
\end{align}
\noindent where $p(\cdot)$ is the unique probability distribution corresponding to the kernel $k(\cdot,\cdot)$,
assuming the kernel is properly scaled. We use this characterization to obtain local group invariant features
as follows:
\begin{align*}
& k_{q,G}(x,x') \\
&= \int_G \int_G E_{\omega\sim p} \left[e^{-i (gx-g'x')^\top \omega}\right]  q(g) q(g') d\nu(g) d\nu(g') \\
&= E_{\omega\sim p}  \int_G \int_G e^{-i (gx-g'x')^\top \omega} q(g) q(g') d\nu(g) d\nu(g') \\
&= E_{\omega\sim p}  \int_G e^{-i \la\omega, gx\ra} q(g) d\nu(g) \int_G e^{i \la\omega, g'x'\ra} q(g') d\nu(g')  \\
&\approx E_{\omega\sim p} \frac{1}{r^2} \sum_{k=1}^r e^{-i \la\omega, g_k x\ra} \sum_{k=1}^r e^{i \la\omega, g_k x\ra}, \quad (g_k \sim q) \\
&\approx \frac{1}{s r^2} \sum_{j=1}^s \sum_{k=1}^r e^{-i \la\omega_j, g_k x\ra} \sum_{k=1}^r e^{i \la\omega_j, g_k x\ra}, \,\, (g_k \sim q, \omega_j\sim p) \\
&:= \langle \psi_{RF}(x), \psi_{RF}(x')\rangle_{\mbb{C}^s}, \stepcounter{equation}\tag{\theequation}\label{eq:rf_feats}
\end{align*}
\noindent where 
\begin{align}
\begin{split}
\psi_{RF}(x) &= \frac{1}{r\sqrt{s}} \left[ \sum_{k=1}^r e^{-i \la\omega_1, g_k x\ra} \ldots \sum_{k=1}^r e^{-i \la\omega_s, g_k x\ra} \right] \\
& \in\mbb{C}^s.
\end{split}
\label{eq:rf_feats1}
\end{align}
We use standard Monte Carlo to approximate both inner integral over the group and the outer expectation over $\omega$.
It is also possible to use quasi Monte Carlo approximation for the expectation over $\omega$, 
which has been carefully studied for random Fourier features~\cite{quasimc2014}.
We provide uniform convergence bounds and excess risk bounds for these features in Section \ref{sec:theory}.  

The feature map $\psi_{RF}(\cdot)$ requires us to apply $r$ group actions to every data point which can be expensive
in large data regime. If the group action is unitary transformation preserving norms and distances between points 
(i.e., $\lVert gx\rVert_2 = \lVert x\rVert_2$), the inner product satisfies $\la x,x'\ra = \la gx, gx'\ra$.
This can be used to transfer the group action from the data to the sampled template 
as $\la \omega,gx\ra = \la g^{-1}\omega, g^{-1}gx\ra = \la g^{-1}\omega, x\ra$ \cite{poggio2014inv}
without affecting the approximation of kernel $k_{q,G}$, as long as the pdf $q$ is symmetric around the identity element
($q(g)=q(g^{-1})\, \forall g\in G$).
For instance, in the case of images which can be viewed as a function $I:\reals^2\mapsto\reals$, 
one can show the following result\footnote{This is mentioned in \cite{poggio2014inv}
as a remark without a formal proof. We provide a proof in the appendix for completeness.}
regarding group actions (e.g., rotation, translation, scaling, affine transformation).
\begin{lemma}
Let $g$ be a group element acting on an image $I:\reals^2\mapsto\reals$. 
The group action defined as $T_g [I(x)] = \lvert J_g\rvert^{-1/2} I(g^{-1} x),\,\forall\, x$, where $J_g$ is the
Jacobian of the transformation, is a unitary transformation
and satisfies $\la T_g(I), T_g(I')\ra = \la I,I'\ra$. 
\label{lem:unitary}
\end{lemma}
\begin{proof}
See appendix.
\end{proof}
The lemma suggests scaling the pixel intensities of the image by a factor $\lvert J_g\rvert^{-1/2}$
to make the group action unitary. The Jacobian for rotating or translating an image has determinant $1$
obviating the need for scaling. For general affine transformation, we need to scale the pixel
intensities accordingly to keep it unitary\footnote{The Jacobian for affine transformation $T(x)=Ax+b$ 
is its linear component $A$.}.

\vspace{1mm}
\subsubsection{Features using Nystr\"{o}m approximation}
\label{subsubsec:nystrom}
\vspace{1mm}
Here we consider the case of a general base kernel and derive local group invariant 
features using Nystr\"{o}m approximation \cite{nystrom2001,nystrom2005}. 
Nystr\"{o}m method starts with identifying a set of \emph{landmark points} (also referred as \emph{templates}) $Z=\{z_1,\ldots,z_s\}$ 
and approximates each function $f\in\mcal{H}$ by its orthogonal projection
onto the subspace spanned by $\{k(\cdot,z_i)\}_{i=1}^s$. Several schemes for identifying
the landmark points have been studied in the literature, including random sampling,
sampling based on leverage scores, and clustering based landmark selection \cite{nystrom2012sampling,gittens2013nystrom}.
We can choose landmarks from the original set $X$ or from the orbit $gX$.
Nystr\"{o}m method approximates the kernel as $k(x,x')\approx K_{Z,x}^\top K_{Z,Z}^{+} K_{Z,x'}$, where
$K_{Z,x}=[k(x,z_1),\ldots,k(x,z_s)]^\top$ and $K_{Z,Z}$ is square kernel matrix for the landmark points
with $K_{Z,Z}^+$ denoting the pseudo-inverse. 

Since $K_{Z,Z}$ is a positive semi-definite matrix, let $K_{Z,Z}^+ = L^\top L$, where $L\in\reals^{\text{rank}(K_{Z,Z})\times s}$. 
We have
\begin{align*}
& k_{q,G}(x,x') \\
&\approx \int_G \int_G K_{gx,Z} K_{Z,Z}^+ K_{Z,g'x'}\, q(g) q(g') d\nu(g) d\nu(g') \\
&= \int_G \int_G K_{gx,Z} L^\top L K_{Z,g'x'}\, q(g) q(g') d\nu(g) d\nu(g') \\
&= \left\langle \int_G  L K_{Z,gx}\, q(g) d\nu(g), \int_G  L K_{Z,g x'}\, q(g) d\nu(g) \right\rangle \\
&\approx \left\langle L \frac{1}{r}\sum_{k=1}^r  K_{Z,g_k x}\, , L \frac{1}{r}\sum_{k=1}^r K_{Z,g_k x'} \right\rangle , \quad (g_k \sim q),
\end{align*}

\noindent where the features are given by 
\begin{align}
\psi_{Nys}(x) = \frac{1}{r} L \sum_{k=1}^r K_{Z,g_k x}\,\, \in \reals^{\text{rank}(K_{Z,Z})}.
\end{align}
If the base kernel satisfies $k(gx,gx')=k(x,x'),\,\forall\, g,x,x'$, we can transfer the group action from the data points
to the landmark points as $k(gx,z)=k(g^{-1}gx,g^{-1}z) = k(x,g^{-1}z)$
without affecting the Nystr\"{o}m approximation of $k_{q,G}$, as long as the pdf $q$ is symmetric around the identity element
($q(g)=q(g^{-1})\, \forall g\in G$).
This becomes essential in large data regime where the number of data points is much larger than the number of landmarks. 
For the group action defined in Lemma~\ref{lem:unitary}, 
all dot product kernels ($\tilde{k}(\la x,x'\ra)$) and shift invariant kernels ($\tilde{k}(\lVert x-x'\rVert_2)$)
satisfy this property. 

\noindent {\bf Remarks:} \\
{\bf (1)} Earlier work \cite{poggio2014inv,mroueh2015} 
has proposed features of the form $\phi_n^k(x)=1/r \sum_{j=1}^r \eta_n(\langle g_j x, \omega_k\rangle)$ where $\eta_n(\cdot)$
were taken to be step functions $\eta_n(a)={1}(a<h_n)$  
with preselected thresholds $h_n$.  
Nonlinearities in our proposed local invariant features 
emerge naturally as a result of kernel approximation, with $\eta(x,\omega) = e^{-i\la x,\omega\ra}$ for
$\psi_{RF}$ and $\eta(x,\omega) = k(x,\omega)$ for $\psi_{Nys}$. \\
{\bf (2)} Our work can also be viewed as incorporating local group invariance in widely used random Fourier
and Nystr\"{o}m approximation methods, however this viewpoint overlooks the Hilbert space embedding perspective motivated in this work. \\
{\bf (3)} The kernel $k_{q,G}$ defined in Eq.~\eqref{eq:localhaar} assumes a linear hyperkernel $h:\mcal{H}\times\mcal{H}\mapsto\reals$
over RKHS embeddings of orbit distributions. It is also possible to use a nonlinear hyperkernel along the lines of 
\cite{christmann2010universal} and \cite{muandet2012learning}, and approximate it using a second
layer of random Fourier (RF) or Nystr\"{o}m features. We show empirical results for both linear and Gaussian hyperkernel 
(approximated using RF features) in Sec.~\ref{sec:emp_eval}. \\
{\bf (4)} {\bf Computational aspects.} The complexity of feature computation is $rC_f + rsC_g$ where $C_f$ is the cost of computing the 
vanilla random Fourier or vanilla Nystr\"{o}m features and $C_g$ is the cost of computing a group action on 
a \emph{template} $\omega$. However same set of templates are used for all data points so group actions 
on the templates can be computed in advance. Structured random Gaussian templates can also be used in our framework to
speed up the computation of random Fourier features $\psi_{RF}$ \cite{le2013fastfood,choromanski2016recycling,bojarski2016structured}. 
Recent approaches for scaling randomized kernel machines to massive data sizes and very large number of 
random features can also be used \cite{avron2015high}.  \\


\section{Theory} \label{sec:theory}

In this section we focus on local  invariance  learning  using the random feature map $\psi_{RF}$ defined in Section \ref{subsubsec:fourier} for the Gaussian base kernel $k(\cdot,\cdot)$. 
We first address the uniform convergence of the random feature map $\psi_{RF}$ to  the local invariant kernel $k_{q,G}$ on a set of points $\mathcal{M}$. 
In other words we show in Theorem \ref{thm:ker_approx_fourier}  that for a sufficiently large number of random templates $s$, and group element samples $r$,   we have $\langle \psi_{RF}(x),\psi_{RF}(y)\rangle \approx k_{q,G}(x,y)$, for all points $x,y \in \mathcal{M}$.
Second we consider a supervised binary classification setting, and study  generalization bounds of learning a linear classifier in the local invariant random feature space $\psi_{RF}$. In a nutshell Theorem \ref{theo:GBounds} shows that linear functions in the random feature space $\langle w,\psi_{RF}(x)\rangle$, approximate functions in the RKHS induced by our local invariant kernel $k_{q,G}$.      

\subsection{Uniform Convergence}
Theorem \ref{thm:ker_approx_fourier}  provides a uniform convergence bound of our invariant random feature map $\psi_{RF}$ for Gaussian base kernel $k(\cdot,\cdot)$.
\begin{theorem}[Uniform convergence of Fourier Approximation] \label{thm:ker_approx_fourier}
 Let $X$ be a compact space in $\mathbb{R}^d$ with diameter $\diam(X)$. 
 For $\varepsilon>0, \delta_1,\delta _2 \in (0,1)$,  the following uniform convergence bound holds with probability $1- \Big(\frac{64(d+1)}{\varepsilon^2 \sigma^2}\Big)^{\frac{d}{d+1}}(\delta_1+\delta_2)^{\frac{2d}{d+1}}.$
  \begin{align*}
 \underset{x,y \in X}{\sup}\Big |\Big\langle \psi_{RF}(x),\psi_{RF}(y) \Big \rangle- K_{q,G}(x,y)\Big| \leq \varepsilon + \frac{1}{r}
 \end{align*}
for a number of group samples $$r \geq  C_1\frac{d}{\varepsilon^2} \log(\rm{diam}(X)/\delta_1),$$
and a number of random templates
$$s \geq C_2 \frac{d}{\varepsilon^2}\log(\rm{diam}(X)/\delta_2),$$
 where $\sigma_p^2 \equiv E_p[\omega^\top \omega] = d/\sigma^2$ is the second moment of the Fourier transform of the Gaussian base kernel $k$, and $C_1$ and $C_2$ are numeric universal constants. \\
\end{theorem}
\begin{proof}
See Appendix.
\end{proof}

\subsection{Generalization Bounds}
Given a labeled training set $S \ = \big \{  (x_i,y_i)\, |\, x_i \in X, y_i \in Y = \{+1,-1\}\big \}$, our goal is to learn  a decision function $f : \ X\rightarrow Y$  via  empirical risk  minimization (ERM)  $$\min_{f \in \mathcal{H}_{\mathcal{K}}} \hat{\mathcal{E}}_V(f) = \frac{1}{N}\sum_{i = 1}^N V(y_i f(x_i))$$ where $V$ is convex and $L$-Lipschitz  loss function. Let $\mathcal{E}_V(f) = \mbb{E}_{x,y\sim P} V(y f(x))$ be the expected risk for $f\in \mcal{H_K}$. According to the representer theorem, the solution of ERM  is given  by $ f^\star(\cdot) = \sum_{i = 1}^N \alpha_i^\star k_{q,G}(x_i,.)$. 

We consider linear hyperkernel $h$  in Eq. \eqref{eq:localhaar}  and consider  $\mathcal{H}_{\mathcal{K}}$, the  
RKHS induced by the kernel  
$k_{q,G}(x,y) = \int_G \int_G k(gx,g'x')\, q(g) q(g')\, d \nu(g) d \nu(g') $, 
as introduced in Sec. \ref{subsec:orbitembed}.
Similar to \cite{mroueh2015}, for $C>0$, we define $\mathcal{F}_p$ an infinite dimensional space to approximate $\mathcal{H}_\mcal{K}$ (see \cite{rahimi2008uniform} for a motivation for this approximation):
\begin{align*}
\mathcal{F}_p \equiv \Big \{ f(x) = \int_\Omega \alpha(\omega) \int_G \phi(gx,\omega) q(g)d\nu(g) d\omega \Big | \\
 |\alpha(\omega) | \leq C p(\omega)\Big \}, 
\end{align*}
where $\phi(gx,\omega)=e^{-i\la gx,\omega\ra}$. Similarly define the linear space in the span of $\psi_{RF}(\cdot)$, 
$\mathcal{\hat{F}}_p \equiv \Big \{ \hat{f}(x) = \langle \alpha, \psi_{RF}(x)\rangle = \sum_{k=1}^s \alpha_k \frac{1}{r}\sum_{i=1}^{r} \phi(g_i x,\omega_k)  \Big | |\alpha_k | \leq \frac{C}{s}\Big \}$.
\begin{theorem}
Let $\delta >0$.Consider the training set  $S \ = \ \big \{  (x_i,y_i) \ | \ x_i \ \in X, y_i \ \in Y, i = 1 \ldots N  \big \}$ sampled from the input space and let $f_N^\star$ is the empirical risk minimizer such that $f_N^\star \ = \argmin_{f \in \hat{\mathcal{F}}} \hat{\mathcal{E}}_V(f) = \frac{1}{N} \sum_{i= 1}^N V (y_i f(x_i))$, then we have with probability $1-3\delta$ (over the training set, random templates and group elements)
\begin{align*}
&{\mathcal{E}_V}(f^\star_N ) -  \min_{f \in \mathcal{F}_P} {\mathcal{E}_V}(f) \leq \\ & \mcal{O} \left(\left(\frac{1}{\sqrt{N}} + \frac{1}{\sqrt{s}} + \frac{1}{\sqrt{r}} \right)LC \sqrt{\log \frac{1}{\delta}}\right).	
\end{align*}
\label{theo:GBounds}
\end{theorem}
\begin{proof}
See Appendix.
\end{proof}

\section{Empirical Observations} \label{sec:emp_eval}
We evaluate the proposed method (referred as \proposed~ here) on three real datasets. We use Gaussian kernel as the base kernel in all our experiments. 
For methods that produce random (unsupervised) features, 
which include the proposed approach as well as regular random Fourier (abbrv. as RF) \cite{randomfourier07} 
and Nystrom \cite{nystrom2001} method, we report performance with: (i) linear decision boundary on these features 
(linear SVM or linear regularized least squares (RLS)),
and (ii) nonlinear decision boundary which is realized by having a Gaussian kernel on top of the features and approximating it through 
random Fourier features \cite{randomfourier07}, followed by a linear SVM or RLS. The later can also be viewed as
using a nonlinear hyperkernel over RKHS embeddings of orbit distributions (also see Remark (3) at the end of Sec.~\ref{sec:formulation}). 
Parameters for all the methods are selected using grid search on a hold-out validation set unless otherwise stated. 

\begin{table}[t]
\centering
\begin{tabular}{l c c}
\toprule
 & \textbf{RMSE}& \textbf{RMSE w/} \\
\textbf{Method } & & {\bf 2nd layer RF}  \\
\midrule
Original (RF) & 14.01 & 13.78 \\
Original (Nys) & 13.97 & 13.81  \\
Original (GP) & 13.48 & \scriptsize N/A \\
Sort-Coulomb (RF) & 12.89 & 12.49  \\
Sort-Coulomb (Nys) & 12.83 & 12.51  \\
Sort-Coulomb (GP)\cite{montavon2012learning} & 12.59 & \scriptsize N/A \\
Rand-Coulomb\cite{montavon2012learning} & 11.40 & \scriptsize N/A   \\
GICDF \cite{mroueh2015} & 12.25 & \scriptsize N/A \\
\proposed (RF) & \bf 10.82 & \bf 10.05 \\
\proposed (Nys) & 10.87 & 10.45 \\
\bottomrule
\end{tabular}
\caption{RMSE on Quantum Machine data}
\label{tab:quant1}
\end{table}

\subsection{Quantum Machine dataset}  
This data consists of $7165$ Coulomb matrices of size $23\times 23$ (each matrix corresponding to a molecule)
and their associated atomization energies in kcal/mol. It is a small subset of a large dataset collected by 
Blum and Reymond (2009) \cite{blum2009970}, and was recently used by Montavon et al. (2012) \cite{montavon2012learning} for evaluation.
The goal is to predict atomization energies of molecules which is modeled as a regression task.

The atomization energy is known to be invariant to permutations of rows/columns of the Coulomb matrix
which motivates the use of representations invariant to the permutation group.
We follow the experimental methodology of \cite{montavon2012learning} and report mean cross-validation
accuracy on the five folds provided in the dataset. An inner cross-validation is used for tuning the parameters
for each fold as in \cite{montavon2012learning}. 
We compare the performance of our method with several baselines in Table \ref{tab:quant1}: 
(i) \emph{Original (GP/RF/Nys):} Gaussian Process regression on original Coulomb matrices and its approximation via random Fourier (RF) \cite{randomfourier07} and Nystrom features \cite{nystrom2001},
(ii) \emph{Sort-Coulomb (GP/RF/Nys):} GP regression on sorted Coulomb matrices (sorted according to row norms) \cite{montavon2012learning} and its approximation,
(iii) \emph{Rand-Coulomb:} permutation invariant kernel proposed in \cite{montavon2012learning}, and
(iv) \emph{GICDF:} Group invariant CDF (histogram) based features proposed in \cite{mroueh2015}.
The results for \emph{Sort-Coulomb (GP)} and \emph{Rand-Coulomb} are taken directly from \cite{montavon2012learning}. 
For all RF and Nystr\"{o}m based features we use $10k$ random templates ($\omega$).
For GICDF and our method, we sample $70$ random permutations ($r=70$ in Eq.~\ref{eq:rf_feats1}) using the same scheme as in \cite{montavon2012learning}.
The proposed \proposed~ outperforms all these directly competing methods including \emph{Rand-Coulomb} and \emph{GICDF}. 
Neural network based features used in \cite{montavon2012learning} can also be used within our framework
but we stick to raw Coulomb matrices for simplicity sake.


\begin{figure}[h] \label{fig:fig2}
\centering
\includegraphics[width=8.5cm]{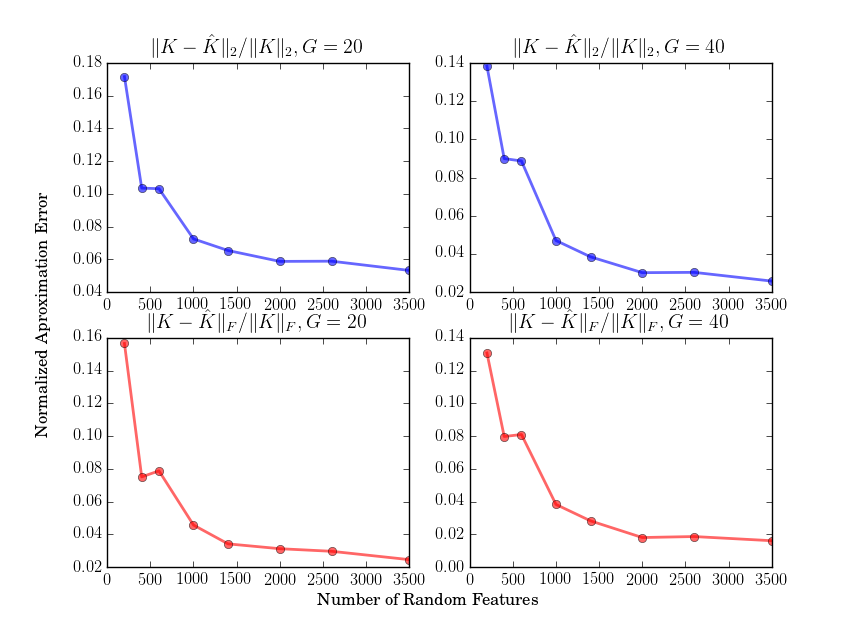}
\vspace{-30pt}
\caption{Kernel approximation error (normalized) in spectral and Frobenius norms vs number of random features, for $20$ (left) and $40$ (right) group transformations}
\end{figure}

{\bf Kernel Approximation.} 
We also report empirical results on approximation error for kernel matrix (in terms of spectral norm and Frobenius norm)
in Fig.~\ref{fig:fig2}. The plots show the approximation error for different number of group actions as the number of random
Fourier features are increased. The kernel used is the Gaussian kernel. 
The true kernel has been computed using $70$ group elements randomly sampled from the permutation group.
The normalized error for all the cases goes down with the number of random Fourier features which is in line with our theoretical convergence results.

\begin{table}[t]
\label{tab:mnist}
\centering
\begin{tabular}{l c c}
\toprule
 & \textbf{Accuracy}& \textbf{Accuracy w/} \\
\textbf{Method } & & {\bf 2nd layer RF}  \\
\midrule
Original (RF) & 87.75 & 88.01 \\
Original (Nys) & 88.93 & 88.98 \\
Original (RBF) & 90 & \scriptsize N/A \\
TI-RBM \cite{sohnlee12} & 95.8 & \scriptsize N/A \\
RC-RBM \cite{schmidt12} & 97.02 & \scriptsize N/A \\
GICDF \cite{mroueh2015} & 93.81 & \scriptsize N/A \\
Z2-CNN \cite{groupequiconvnets} & 94.97 & \scriptsize N/A \\
P4-CNN \cite{groupequiconvnets} & \bf 97.72 & \scriptsize N/A \\
\proposed (RF) & \bf 96.83 &  97.18 \\
\proposed (Nys) & 96.81 & \bf 97.21 \\
\bottomrule
\end{tabular}
\caption{Rotated MNIST results}
\end{table}

\subsection{Rotated MNIST} 
Rotated MNIST dataset \cite{larochelle2007empirical} consists of total $62k$ images of digits ($12k$ for training
and $50k$ for test), obtained by
rotating original MNIST images by an angle sampled uniformly between $0$ and $2\pi$. 
We compare the proposed method with several other approaches in Table \ref{tab:mnist}.
We use von-misses distribution ($p(\theta)=\exp(-\kappa \cos (\theta))$ with $\kappa=0.2$, selected using cross-validation) to 
sample $r=100$ rotations. 
We use $s=7k$ random templates for both RF and Nystrom approximations, and use $17k$ random templates 
for layer-2 RF approximation. 
The results for the cited methods in Table \ref{tab:mnist} are directly taken from the respective papers,
except for GICDF \cite{mroueh2015} which we implemented ourselves. 
The proposed \proposed~ outperforms most of the competing methods including
deep architectures like rotation-invariant convolutional RBM (RC-RBM) \cite{schmidt12},
transformation invariant RBM (TI-RBM) \cite{sohnlee12}, and regular deep CNN (Z2-CNN) \cite{groupequiconvnets}.
Our method also performs close to the recently proposed group-equivariant CNN (P4-CNN) \cite{groupequiconvnets}.

\vspace{-2mm}
\subsection{CIFAR-10} 
\vspace{-2mm}
The CIFAR-10 dataset consists of $60k$ RGB images ($50k/10k$ for train/test) of size $32 \times 32$, divided into 10 classes. 
We consider a sub-group of the affine group $Aff(2)$ consisting of rotations, translations and isotropic scaling. 
Instead of operating with a distribution (e.g. Gaussian) over this subgroup, 
we use three individual distributions to have better control over the three variations: 
a log-normal distribution over the scaling group ($\mu=0,\sigma=0.3$), 
a Gaussian distribution over the translation group ($\mu=0,\sigma=0.3$),
and a von-misses distribution over the rotation group ($\kappa=9$).
We observe that working with wider distributions over these groups actually hurts the performance,
highlighting the importance of locality for CIFAR-10. 
We use the normalized pixel intensities as our input features and use the group action defined in Lemma \ref{lem:unitary} to keep it unitary. 
We use $s=10k$ random templates and $r=50$ group transforms for the first layer RF features (Eq.~\ref{eq:rf_feats1}), and 
use $30k$ random templates for second layer RF features. 
The proposed \proposed~outperforms vanilla RF features as shown in Table \ref{tab:cifar}. 
Nystr\"{o}m based features gave similar results as random Fourier features in our early 
explorations. 
We were not able to scale GICDF~\cite{mroueh2015} to a suitable number of random templates
due to memory issues (for every random template, GICDF generates $n$ features (number of bins, set to $25$ following \cite{mroueh2015})
blowing up the overall feature dimension to $n\times 10k$).
Note that the performance of \proposed~on this data is still significantly worse than 
deep CNNs \cite{groupequiconvnets} since \proposed~treats the image as a vector ignoring
the spatial neighborhood structure taken into account by CNNs 
through translation invariance over small image patches. Incorporating orbit statistics of image patches
in our framework is left for future work.

\begin{table} \label{tab:cifar}
\centering
\begin{tabular}{c c c  c  }
\toprule
\multicolumn{2}{c}{~~~Original (RF) \cite{randomfourier07}~~~}  & \multicolumn{2}{c}{~~~\proposed~~~} \\
\textbf{1-layer} & \textbf{2-layer} & \textbf{1-layer} & \textbf{2-layer} \\
\midrule
 61.02 & 62.79 & \textbf{64.19} & \textbf{67.32}  \\
\bottomrule
\end{tabular}
\caption{CIFAR-10 results}
\end{table}

\section{Related Work}
\textbf{Invariant Kernel Methods.}  \cite{Anselmi_IAI}  introduced Tomographic Probabilistic Representations (TPR) that embed orbits to probability distributions. Unlike TPR, our representation maps  orbits or local portions of the orbit via kernel mean embedding to an RKHS and allows to define similarity between orbits in this space. Indeed our representation is infinite dimensional and is related to Haar Invariant Kernel \cite{haarintkernels}. As discussed earlier it can be approximated via random features or Nystr\"{o}m sampling. Other approaches for building invariant kernels were defined in \cite{WChapelle2007} that focuses on dilation invariances.  A kernel view of histogram of gradients was introduced in \cite{KernelDesc}, where finite dimensional features were defined through kernel PCA. Kernel convolutional networks introduced in  \cite{Mairal1},\cite{Mairal2}, considers the composition of  multilayer kernels, where local image patches are represented as points in a reproducing kernel. However they do not consider general group invariances. The work of \cite{dai2016learning} considers the general problem of learning from conditional distributions. When applied to invariant learning, their optimization approach needs to sample a group transformed example in every SGD iteration whereas our approach allows working with group actions on the random templates. \\
 \textbf{Invariance in Neural Networks.} Inducing invariances in neural networks has attracted many recent research streams. It is now well established that convolutional neural networks (CNN) \cite{Lecun98gradient-basedlearning} ensure translation invariance.  \cite{Galaxy} showed that mapping orbits of rotated and flipped images through a shared fully connected network builds some invariance in the network. Scattering networks \cite{Bruna}  have built in invariances for the roto-translation group.   \cite{DSN} generalizes CNN to general group transformations. \cite{dieleman-cyclic-2016} exploits cyclic symmetry to have invariant prediction in the network.  More recently, \cite{groupequiconvnets} designs a convolutional neural network that is equivariant to group transforms by introducing convolution over the group.

\section{Concluding Remarks}
\vspace{-2mm}
The proposed approach can be suitable for large-scale problems, benefiting from the recent advances in scalability of randomized kernel methods~\cite{lu2014scale,dai2014scalable,avron2015high}. 
As a future direction, we would like to extend our framework to operate at the level of image patches, enabling us to capture local spatial
structure.
Further, the proposed approach requires computation of all $r$ group transformations for all the sampled random templates. Reducing the required number of group transformations is an important direction for future work. 
Our work also assumes that the appropriate group actions are given. Extension to the case when the group transformations are learned from the data (e.g., using local tangent space \cite{rifai2011manifold}) is also an important direction for future work.

{ {\bf Acknowledgments:} We thank Dmitry Malioutov for several insightful discussions. This work was done while Anant Raj was a summer intern at IBM Research. }
\newpage
\clearpage
\sloppy
\bibliography{ml}
\bibliographystyle{plain}

\appendix

\setcounter{section}{1}

\onecolumn
\begin{center}
{\centering \LARGE Appendix }
\vspace{1cm}
\end{center}
\begin{proof}[\bf Proof of Lemma 2.1]
Since unitary transformations preserve dot-products, i.e., $\la T(x), T(y)\ra = \la x,y\ra$, we need to show
that a group element acting on the image $I:\reals^2\mapsto\reals$ as $T_g [I(x)] = \lvert J_g\rvert^{-1/2} I(T_g^{-1} (x)),\,\forall\, x$ is a unitary
transformation. 

Let $J_g$ be the Jacobian of the transformation $T_g$, with determinant $\lvert J_g\rvert$. We have
\begin{align*}
\lVert I(T_g^{-1} (\cdot))\rVert^2 &= \int  I^2 (T_g^{-1} (x)) dx \\
&= \int I^2(z)\, \lvert J_g\rvert\, dz, \quad \text{substituting } z = T_g^{-1}(x) \Rightarrow dx = \lvert J_g\rvert\, dz \\
&= \lvert J_g\rvert\, \lVert I(\cdot)\rVert^2
\end{align*}
Hence the transformation given as $T_g [I(\cdot)] = \lvert J_g\rvert^{-1/2} I(T_g^{-1} (\cdot))$ is unitary 
and thus $\la T_g(I), T_g(I') \ra = \la I, I'\ra$ for two images $I$ and $I'$. 
\end{proof}
\begin{proof}[\bf Proof of Theorem 3.1]  

We first define the notion of \emph{U-statistics} \cite{hoeffding1963probability}. \\

\textbf{U-statistics -} Let $g : \mathbb{R}^2 \rightarrow \mathbb{R}$ be a symmetric function of its arguments. Given an i.i.d. sequence $X_1, X_2 \cdots X_k$ of $k (\geq 2)$ random variables, the quantity $U:= \frac{1}{n(n-1)} \sum_{i\neq j, i,j = 1}^n g(X_i, X_j)$ is known as a pairwise \textbf{U-statistics}. If $\theta(P) = \mathbb{E}_{X_1,X_2 \sim P}\, g(X_1, X_2)$ then $U$ is an unbiased estimate of $\theta(P)$. 

Our goal is to bound 
$$\underset{x,y \in X}{\sup}\Big |\Big\langle \psi_{RF}(x),\psi_{RF}(y) \Big \rangle- k_{q,G}(x,y)\Big|$$
\noindent where 
$$ \psi_{RF}(x)=\frac{1}{r} \sum_{i=1}^{r} z(g_i x), x \in X \subset \mathbb{R}^d.$$ 
\noindent We work with $z(\cdot)=\sqrt{2/s}[\cos(\langle \omega_1, \cdot\rangle + b_1),\ldots,  cos(\langle \omega_s, \cdot\rangle + b_s]\in \mathbb{R}^s$ with $b_i \sim \text{Unif}(0,2\pi)$ as in \cite{randomfourier07}.

\noindent Let $\widehat{k}_{q,G}(x,y) :=\frac{1}{r^2} \sum_{i,j=1}^{r^2} k (g_i x, g_j y)$ and $\widetilde{k}_{q,G}(x,y) :=\frac{1}{r(r-1)} \sum_{i \neq j  , i,j = 1}^{r^2} k (g_i x, g_j y)$. \\

\noindent Using the triangle inequality we have
\begin{align*}
& \underset{x,y \in X}{\sup}\Big |\Big\langle \psi_{RF}(x),\psi_{RF}(v) \Big \rangle-  k_{q,G}(x,y)\Big|  \leq \underbrace{\underset{x,y \in X}{\sup}\Big |\Big\langle \psi_{RF}(x),\psi_{RF}(y) \Big \rangle- \widehat{k}_{q,G}(x,y)\Big| }_{A}  \\
&\qquad \qquad \qquad \qquad \qquad \qquad  
+ \underbrace{\underset{x,y \in X}{\sup}\Big | \widetilde{k}_{q,G}(x,y) - k_{q,G}(x,y)\Big|}_{B}   
+ \underbrace{\underset{x,y \in X}{\sup}\Big |\widehat{k}_{q,G}(x,y)- \widetilde{k}_{q,G}(x,y)\Big|}_{C} 
\end{align*}

\noindent \textbf{Bounding A.} 
\begin{align*}
A &:= \sup_{x ,y \in X}\Big |\frac{1}{r^2}\sum_{i,j} \left(\langle z(g_i x),z(g_j y) \rangle- k(g_i x,g_j y)\right)\Big|\\
\end{align*}
Let us define $f_{ij}(x,y) := \langle z(g_i x),z(g_j y) \rangle- k(g_i x,g_j y)$, and $f(x,y) = 1/r^2 \sum_{i,j} f_{ij}(x,y)$. Since each of the $s$ independent random variables in the summand of $1/r^2 \sum_{i,j} \langle z(g_ix), z(g_jy)\rangle = \frac{1}{s} \sum_{k=1}^s \bigg(\frac{1}{r^2}$ $\sum_{i,j} 2 \cos(\langle \omega_k, g_i x\rangle + b_k)$ $ \cos(\langle \omega_k, g_j y\rangle + b_k)  \bigg)$ is bounded by $[-2,2]$, using Hoeffding's inequality for a given pair $x,y$, we have

$$ \text{Pr} [|f(x,y)| \geq \varepsilon/4] \leq 2 \exp (-s\varepsilon^2/128).$$  

To obtain a uniform convergence guarantee over $X$, we follow the arguments in \cite{randomfourier07}, relying on covering the space with an $\varepsilon$-net and Lipschitz continuity of the  function $f(x,y)$. 

Since $X$ is compact, we can find an $\varepsilon$-net that covers $X$ with $N_X =\left(\frac{2\,\rm{diam}(X)}{\eta}\right)^{d}$ balls of radius $\eta$ \cite{cuckersmale02}.
Let $\{c_k\}_{k=1}^{N_X}$ be the centers of these balls, and let $L_f$ denote the Lipschitz constant of $f(\cdot,\cdot)$, i.e., 
$|f(x,y) - f(c_k,c_l)| \leq L_f \lVert \left(\begin{smallmatrix} x \\ y \end{smallmatrix}\right) - \left(\begin{smallmatrix} c_k \\ c_l \end{smallmatrix}\right)\rVert$ for all $x,y,c_k,c_l \in X$.
For any $x,y\in X$, there exists a pair of centers $c_k,c_l$ such that $\lVert \left(\begin{smallmatrix} x \\ y \end{smallmatrix}\right) - \left(\begin{smallmatrix} c_k \\ c_l \end{smallmatrix}\right)\rVert < \sqrt{2} \eta$.
We will have $|f(x,y)| <\varepsilon/2$ for all $x,y$ if (i) $|f(c_k,c_l)| < \frac{\varepsilon}{4},\,\forall c_k, c_l$, and (ii) $L_f < \frac{\varepsilon}{4\sqrt{2}\eta}$.

We immediately get the following by applying union bound for all the center pairs $(c_k,c_l)$ 
\begin{align} \text{Pr} \left[\cup_{k,l} |f(c_k,c_l)| \geq \varepsilon/4\right] \leq  2\, N^2_X \exp (-s\varepsilon^2/128).
\label{eq:unionballs}
\end{align}

We use Markov inequality to bound the Lipschitz constant of $f$. By definition, we have $L_f = \sup_{x,y} \lVert \nabla_{x,y}$ $f(x,y)\rVert = \lVert \nabla_{x,y} f (x^*,y^*)\rVert$, where $\nabla_{x,y} f(x,y) = \left(\begin{smallmatrix} \nabla_{x} f (x,y) \\ \nabla_{y} f (x,y) \end{smallmatrix}\right)$.
We also have $\mbb{E}_{\omega\sim p} \nabla_{x,y} \langle z($ $g_i x), z(g_j y) \rangle = \nabla_{x,y} k(g_i x,g_j y)$. It follows that
\begin{align*}
\mbb{E}_{\omega\sim p} \left\lVert \nabla_{x,y} f(x^*,y^*)\right\rVert^2 & = \mbb{E}_{\omega\sim p} \left\lVert \frac{1}{r^2}\sum_{i,j=1}^r \nabla_{x,y} \langle z(g_i x^*),z(g_j y^*) \rangle\right\rVert^2 - \left\lVert \frac{1}{r^2}\sum_{i,j=1}^r \nabla_{x,y} k(g_i x^*,g_j y^*)\right\rVert^2 \\
& \leq \mbb{E}_{\omega\sim p} \left\lVert \frac{1}{r^2}\sum_{i,j=1}^r \nabla_{x,y} \langle z(g_i x^*),z(g_j y^*) \rangle\right\rVert^2 \\
& \leq \mbb{E}_{\omega\sim p} \left( \frac{1}{r^2}\sum_{i,j=1}^r \lVert\nabla_{x,y} \langle z(g_i x^*),z(g_j y^*) \rangle \rVert \right)^2 \\
& \leq 2\mbb{E}_{\omega\sim p} \sup_{x,y,g_i,g_j}   \lVert\nabla_{x} \langle z(g_i x),z(g_j y) \rangle \rVert^2 \\
& \leq 2\mbb{E}_{\omega\sim p} \sup_{x,g}   \left( \frac{1}{s}\sum_{k=1}^s \lVert \nabla_x T_g(x) \omega_k  \rVert \right)^2 \\
& \leq 2\mbb{E}_{\omega\sim p} \sup_{x,g}   \left( \frac{1}{s}\sum_{k=1}^s \lVert \nabla_x T_g(x)\rVert_2 \lVert\omega_k  \rVert \right)^2 \\
& = 2\mbb{E}_{\omega\sim p} \sup_{x,g}   \lVert \nabla_x T_g(x)\rVert_2^2  \frac{1}{s^2}\sum_{k=1}^s \sum_{l=1}^s \lVert\omega_k  \rVert \lVert\omega_l  \rVert  \\
& = 2 \sup_{x,g}   \lVert \nabla_x T_g(x)\rVert_2^2 \frac{1}{s^2}\sum_{k=1}^s \sum_{l=1}^s \mbb{E}_{\omega\sim p} \lVert\omega_k  \rVert \lVert\omega_l  \rVert  \\
& = 2 \sup_{x,g}   \lVert \nabla_x T_g(x)\rVert_2^2  \frac{1}{s^2}\left(s \mbb{E}_{\omega\sim p} \lVert\omega \rVert^2 + \sum_{k,l=1,k\neq l}^s (\mbb{E}_{\omega\sim p}  \lVert\omega \rVert)^2  \right) \quad (\text{$\omega_k$ i.i.d.})\\
& \leq 2 \sup_{x,g}   \lVert \nabla_x T_g(x)\rVert_2^2  \frac{1}{s^2}\left(s \mbb{E}_{\omega\sim p} \lVert\omega \rVert^2 + \sum_{k,l=1,k\neq l}^s \mbb{E}_{\omega\sim p}  \lVert\omega \rVert^2  \right) \quad (\text{Jensen's inequality})\\
& \leq 2 \sigma_p^2 \sup_{x\in X,g\in G} \lVert \nabla_x T_g(x)\rVert_2^2,
\end{align*}
\noindent where $\sigma_p^2 = \mbb{E}(\omega^\top \omega)$, and $T_g(x) = gx$ denotes the transformation corresponding to the group action. 
If we assume the group action to be linear, i.e., $T_g(x+y) = T_g(x) + T_g(y)$ and $T_g(\alpha x) = \alpha T_g(x)$, which holds for all group transformations considered in this work (e.g., rotation, translation, scaling or general affine transformations on image $x$; permutations of $x$), we can bound $ \lVert \nabla_x T_g(x)\rVert_2$ as 
\begin{align*}
\lVert \nabla_x T_g(x)\rVert_2 &= \sup_{u:\lVert u\rVert=1} \lVert \nabla_x T_g(x) u \rVert \\
&= \sup_{u:\lVert u\rVert=1} \left\lVert \lim_{h\to 0} \frac{T_g(x+hu) - T_g(x)}{h} \right\rVert \quad (\text{directional derivative of vector valued function } T_g(\cdot))\\
& = \sup_{u:\lVert u\rVert=1} \lVert T_g (u) \rVert = 1 \\
& \quad (\text{since } T_g(\cdot) \text{ is either unitary or is converted to unitary by construction (see Lemma 2.1)})
\end{align*}

Using Markov inequality, $\text{Pr} [L_f^2 \geq \varepsilon] \leq \mbb{E}(L_f^2)/\varepsilon$, hence we get 
$$ \text{Pr} \left[ L_f \geq \frac{\varepsilon}{4\sqrt{2}\eta} \right] \leq \frac{64\sigma_p^2 \eta^2}{\varepsilon^2}.$$

Combining Eq.~\eqref{eq:unionballs} with the above result on Lipschitz continuity, we get 
\begin{align}
\text{Pr} \left[\sup_{x,y} |f(x,y)| \leq \varepsilon/2\right] \geq  1 - 2\, N^2_X \exp (-s\varepsilon^2/128) - \frac{64\sigma_p^2 \eta^2}{\varepsilon^2}.
\end{align}



\noindent \textbf{Bounding B.}  \\

\noindent As defined earlier, $\widetilde{k}_{q,G}(x,y):=\frac{1}{r(r-1)} \sum_{i \neq j, i,j=1}^{r}k (g_ix, g_j y)$. From the result of U-statistics literature \cite{hoeffding1963probability}, it is easy to see that $\mathbb{E}(\widetilde{k}_{q,G}(x,y))= k_{q,G}(x,y)$. \\
Since $g_1, g_2 \cdots g_r$ are i.i.d samples, we can consider $\widetilde{k}_{q,G}(x,y)$ as function of $r$ random variables $(g_1, g_2, \cdots g_r)$. Denote $\widetilde{k}_{q,G}(x,y)$ as $f(g_1, g_2, \cdots g_r)$ . Now if a variable $g_p$ is changed to $g'_{p}$ then we can bound the absolute difference of the changed and the original function. For the rbf kernel, $ | k (g_px, g_j y) - k (g'_{p}x, g_j y)| \leq 1$
\begin{align*}
|f(g_1, g_2, \cdots g_p,\cdots g_r) - f(g_1, \cdots g_{p-1}, g'_{p}, g_{p+1} \cdots g_r)| &= \frac{1}{r(r-1)} \Big|\sum_{j = 1, j \neq p}^r k (g_px, g_j y) - k (g'_{p}x, g_j y)\Big| \\
&\leq \frac{1}{r(r-1)} \sum_{j = 1, j \neq p}^r |k (g_px, g_j y) - k (g'_{p}x, g_j y) |\\
&\leq \frac{(r - 1)}{r(r - 1)} = \frac{1}{r}
\end{align*}
Using bounded difference inequality
$$Pr\Big[ \big| f(g_1, g_2, \cdots g_r) - \mathbb{E}[{f(g_1,g_2 \cdots g_r)}]\big| \geq \frac{\varepsilon}{2} \Big] \leq 2\exp\Big(\frac{-r \varepsilon ^2}{2}\Big). $$

The above bound holds for a given pair $x,y$. Similar to the earlier segment for bounding the first term $A$, we use the $\varepsilon$-net covering of $X$ and Lipschitz continuity arguments to get a uniform convergence guarantee.   
Using a union bound on all pairs of centers, we have
\begin{align} \label{eq:p1}
Pr \Big[ \cup_{k,\ell = 1}^{N_X}  \Big |  {\mathbb{E}}[k(g c_{k},g' c_{\ell})]  - \frac{1}{r(r -1)} \sum_{i,j=1, i \neq j}^{r}k (g_i c_{k}, g_j c_{\ell})\Big|   > \frac{\varepsilon}{2}\Big] \leq 2 N^2_X \exp\Big(\frac{-r \varepsilon ^2}{2}\Big).
\end{align}
In order to extend the bound from the centers $c_{i}$ to all $x\in X$, we use the Lipschitz continuity argument.
Let $$h(x,y)= {k}_{q,G}(x,y) - \widetilde{k}_{q,G}(x,y).$$
Let $L_h$ denote the Lipschitz constant of $h(\cdot,\cdot)$, i.e., 
$|h(x,y) - h(c_k,c_l)| \leq L_h \lVert \left(\begin{smallmatrix} x \\ y \end{smallmatrix}\right) - \left(\begin{smallmatrix} c_k \\ c_l \end{smallmatrix}\right)\rVert$ for all $x,y,c_k,c_l \in X$.
By the definition of $\varepsilon$-net, for any $x,y\in X$, there exists a pair of centers $c_k,c_l$ such that $\lVert \left(\begin{smallmatrix} x \\ y \end{smallmatrix}\right) - \left(\begin{smallmatrix} c_k \\ c_l \end{smallmatrix}\right)\rVert < \sqrt{2} \eta$.
We will have $|h(x,y)| <\varepsilon/2$ for all $x,y$ if (i) $|h(c_k,c_l)| < \frac{\varepsilon}{4},\,\forall c_k, c_l$, and (ii) $L_h < \frac{\varepsilon}{4\sqrt{2}\eta}$.

We will again use Markov inequality to bound the Lipschitz constant of $h$. By definition, we have $L_h = \sup_{x,y} \lVert \nabla_{x,y} h (x,y)\rVert = \lVert \nabla_{x,y} h (x^*,y^*)\rVert$, where $\nabla_{x,y} h(x,y) = \left(\begin{smallmatrix} \nabla_{x} h (x,y) \\ \nabla_{y} h (x,y) \end{smallmatrix}\right)$.
We also have $\mbb{E}_{\omega\sim p} \nabla_{x,y} \widetilde{k}_{q,G}(x,y) = \nabla_{x,y} k_{q,G}(x,y)$. It follows that
\begin{align*}
\mbb{E}_{g_1,\ldots,g_r} \lVert \nabla_{x,y} h(x^*,y^*)\rVert^2 & = \mbb{E}_{g_1,\ldots,g_r} \lVert \nabla_{x,y} \widetilde{k}_{q,G}(x^*,y^*) \rVert^2 - \lVert\nabla_{x,y} k_{q,G}(x^*, y^*)\rVert^2 \\
& \leq \mbb{E}_{g_1,\ldots,g_r} \lVert \nabla_{x,y} \widetilde{k}_{q,G}(x^*,y^*) \rVert^2 \\
& =  \mbb{E}_{g_1,\ldots,g_r} \left\lVert \frac{1}{r(r-1)} \sum_{i\neq j} \nabla_{x,y} k(g_i x^*,g_j y^*) \right\rVert^2. \\
\end{align*}

Noting $T_{g_{i}}(x)=g_i x$, and $k(x,y)=\exp- \frac{1}{2\sigma^2}\Big|\Big|x-y \Big|\Big|^2 $, we have
\begin{align*}
 \nabla_{x}k(g_i x,g_j y)&=  \nabla_{x}k(T_{g_i}(x),T_{g_j}(y)) \\
 &= -\frac{1}{\sigma^2}\nabla_{x}T_{g_i}(x)(g_i x -g_j y) \exp \left(- \frac{1}{2\sigma^2}\Big|\Big|g_i x-g_j y \Big|\Big|^2 \right).
\end{align*}
\noindent Continuing

\begin{align*}
\Big|\Big|   \frac{1}{r(r-1)} \sum_{i\neq j} \nabla_{x,y} k(g_i x, g_j y)\Big|\Big| & \leq   \frac{1}{r(r-1)} \sum_{i\neq j } \Big|\Big| \nabla_{x,y} k(g_i x,g_j y) \Big|\Big| \\
& \leq   \frac{\sqrt{2}}{r(r-1)} \sup_{x} \sum_{i\neq j } \Big|\Big| \nabla_{x}k(g_i x,g_j y)\Big|\Big| \quad (\text{using symmetry of $k(\cdot,\cdot)$}) \\
&=  \frac{\sqrt{2}}{r(r-1)\sigma^2} \sup_{x} \sum_{i\neq j }k(g_i x,g_j y) \Big|\Big| \nabla_{x}T_{g_i}(x)(g_ix-g_jy)\Big|\Big|\\
&\leq  \frac{\sqrt{2}}{r(r-1)\sigma^2}   \sum_{i\neq j} k(g_i x,g_j y) || \nabla_{x}T_{g_i}(x) ||_{2} ||(g_ix-g_jy) || \\
&\leq  \frac{\sqrt{2}e^{-1/2}}{\sigma} \sup_{x\in X,g \in G} || \nabla_{x}T_{g}(x) ||_{2}  \quad (\text{using } \sup_{z \geq 0} z e^{-z^2/(2\sigma^2)} = \sigma e^{-1/2})\\
&\leq  \frac{\sqrt{2}e^{-1/2}}{\sigma}  \quad (\text{using linearity and unitariy of $T_g(\cdot)$ as before})
\end{align*}

It follows that
$$\mathbb{E} ({L_{h}}^2) \leq  \frac{2}{\sigma^2 e}.$$ 

Now using Markov inequality we have
$$\mathbb{P}\left[L_{h}> \sqrt{t} \right]\leq \frac{\mathbb{E}({L_{h}}^2)}{t},$$
Hence we have for $t=\left(\frac{\varepsilon}{4 \sqrt{2}\eta}\right)^2$ ,
$$\mathbb{P}\left[L_{h}>\frac{\varepsilon}{4 \sqrt{2}\eta} \right]\leq \frac{32\eta^2\mathbb{E}((L_{h})^2)}{\varepsilon^2}\leq  \frac{64 \eta^2}{e \sigma^2 \varepsilon^2} , $$

\noindent Hence 
$$ \text{Pr}[B \leq \varepsilon/2)] \geq 1 - 2 (N_X)^2 \exp \left(\frac{-r\varepsilon^2}{2}\right) - \frac{64\eta^2}{e\sigma^2 \varepsilon^2}.$$

\noindent \textbf{Bounding C.} 
\begin{align*}
\Big | \widetilde{k}_{q,G}(x,y) - \widehat{k}_{q,G}(x,y) \Big| &=\Big |  \frac{1}{r(r - 1)}\sum_{i,j=1, i\neq j}^r k(g_i x, g_j y) - \frac{1}{r^2}\sum_{i,j=1}^r k(g_i x, g_j y)  \Big | \\
 & = \Big |  \Big( \frac{1}{r(r-1)} - \frac{1}{r^2}\Big)\sum_{i,j=1, i\neq j}^r k(g_i x, g_j y) - \frac{1}{r^2}\sum_{ i,j = 1, i = j}^r k(g_i x, g_j y) \Big | \\
 &\leq \max \left( \frac{1}{r^2(r-1)} \sum_{i,j=1, i\neq j}^r k(g_i x, g_j y), \frac{1}{r^2}\sum_{ i,j = 1, i = j}^r k(g_i x, g_j y)\right) \quad (\text{since }k(\cdot,\cdot)\geq 0)\\
 & \leq  \frac{1}{r} \quad (\text{as Gaussian kernel $k(\cdot,\cdot)\leq 1$})
\end{align*}

Finally we have
$$\underset{x,y \in X}{\sup}\Big |\Big\langle \psi_{RF}(x),\psi_{RF}(y) \Big \rangle-  k_{q,G}(x,y)\Big| \leq A+ B+C \leq \varepsilon + \frac{1}{r} $$ 
with a probability at least $1- 2 {N_X}^2 \exp\Big( \frac{-s \varepsilon ^2}{128}\Big)- 2 {N_X}^2  \exp\big( \frac{-r \varepsilon ^2}{2}\big) - \Big(\frac{64 \eta^2 d}{\varepsilon^2 \sigma^2}\Big) - \Big(\frac{64\eta^2}{e\varepsilon^2\sigma^2}\Big)$, noting that $\sigma_p^2 = d/\sigma^2$ for the Gaussian kernel $k(x,y) = e^{-\frac{\lVert x-y\rVert^2}{2\sigma^2}}$.

Let
\begin{align*}
p  &= 1- 2 {N_X}^2 \exp\Big( \frac{-s \varepsilon ^2}{128}\Big)- 2 {N_X}^2  \exp\big( \frac{-r \varepsilon ^2}{2}\big) - \Big(\frac{64 \eta^2 d}{\varepsilon^2 \sigma^2}\Big) - \Big(\frac{64\eta^2}{e\varepsilon^2\sigma^2}\Big) \\
& = 1- 2 {\Big(\frac{2 diam(X)}{\eta}\Big)}^{2d} \exp\Big( \frac{-s \varepsilon ^2}{128}\Big)- 2 {\Big(\frac{2 diam(X)}{\eta}\Big)}^{2d}  \exp\big( \frac{-r \varepsilon ^2}{2}\big) - \Big(\frac{64 \eta^2 d}{\varepsilon^2 \sigma^2}\Big) - \Big(\frac{64\eta^2}{e\varepsilon^2\sigma^2}\Big) \\
&\geq 1- 2\eta^{-2d} \Big( \big(2 diam(X)\big)^{2d}\exp\big( \frac{-r \varepsilon ^2}{2}\big) + \big(2 diam(X)\big)^{2d} \exp\Big( \frac{-s \varepsilon ^2}{128}\Big) \Big) - \eta^2 \Big( \frac{64(d+1)}{\varepsilon^2 \sigma^2} \Big).
\end{align*}

The above probability is of the form of $1 - (\kappa_1 + \kappa_2) \eta^{-2d} - \kappa_3 \eta^2$ where $\kappa_1 = 2\big(2 diam(X)\big)^{2d}\exp\big( \frac{-r \varepsilon ^2}{2}\big)$, $\kappa_2 = 2\big(2 diam(X)\big)^{2d} \exp\Big( \frac{-s \varepsilon ^2}{128}\Big) $ and $\kappa_3 =\Big( \frac{64(d+1)}{\varepsilon^2 \sigma^2} \Big) $. Choose $\eta = \Big( \frac{\kappa_1 + \kappa_2}{\kappa_3} \Big)^{\frac{1}{2(d+1)}}$ \\
Hence $p \geq 1 - 2 (\kappa_1 + \kappa_2)^{\frac{1}{d+1}}\kappa_3^{\frac{d}{d+1}} $. \\
For given $\delta_1,\delta _2 \in (0,1)$, we conclude that for fixed constants $C_1,C_2$ , for $$r \geq  \frac{C_1 d}{\varepsilon ^2}  \log(\rm{diam}(X) /\delta_1),$$
$$s \geq  \frac{C_2d}{\varepsilon^2}\left(\log(\rm{diam}(X)/\delta_2\right),$$
we have
$$\underset{x,y \in X}{\sup}\Big |\Big\langle \psi_{RF}(x),\psi_{RF}(y) \Big \rangle- k_{q,G}(x,y)\Big| \leq \varepsilon + \frac{1}{r},$$
with probability $1- \Big( \frac{64(d+1)}{\varepsilon^2 \sigma^2} \Big)^{\frac{d}{d+1}}(\delta_1+\delta_2)^{\frac{2d}{d+1}}.$

\end{proof}

\begin{proof}[\bf Proof of Theorem 3.2]
We give here the proof of Theorem 3.2.
%

\begin{lemma}[Lemma 4 \cite{rahimi2009weighted}] - Let $\textbf{X} \ = \ \{ x_1, x_2 \cdots x_K \}$ be \textit{iid} random variables in a ball $\mathcal{H}$ of radius $M$ centered around the origin in a Hilbert space.  Denote their average by $\overline{\textbf{X}} \ = \ \frac{1}{K}\sum_{i = 1}^K x_i$. Then for any $\delta > 0$, with probability at least $1 - \delta$, 
$$\|\overline{\textbf{X}} - \mbb{E} \mathbb{\overline{\textbf{X}}}\| \leq \frac{M}{\sqrt{K}} \Big( 1 + \sqrt{2\log \frac{1}{\delta}} \Big)$$
\label{lem:empavg}
\end{lemma}

\begin{proof}
For proof, see \cite{rahimi2009weighted}.
\end{proof}

\noindent Now consider a space of functions, 

$$\mathcal{F}_p \equiv \Big \{ f(x) = \int_\Omega \alpha(\omega) \int_G \phi(gx,\omega) q(g) d\nu(g) d\omega \Big | |\alpha(\omega) | \leq C p(\omega)\Big \}, $$
and also consider another space of functions,
$$\mathcal{\hat{F}}_p \equiv \Big \{ \hat{f}(x) = \sum_{k=1}^s \alpha_k \frac{1}{r}\sum_{i=1}^{r} \phi(g_i x,\omega_k)  \Big | |\alpha_k | \leq \frac{C}{s}\Big \}, $$
\noindent where $\phi(gx,\omega) = e^{-i\langle gx,\omega\rangle}$. 
\begin{lemma}
Let $\mu$ be a measure defined on $X$, and $f^\star$  a function in  $\mathcal{F}_p$. If $\omega_1, \omega_2 \ldots \omega_s$ are iid samples from $p(\omega)$, then for $\delta_1,\delta_2 > 0$, there exists a function $\hat{f}\in \mcal{\hat{F}}_p$ such that  $$\Big\| f^\star - \hat{f}\Big \|_{\mathcal{L}_2(X,\mu)}  \leq \frac{C}{\sqrt{s}}\Big( 1 + \sqrt{2\ \log \frac{1}{\delta_1}}\Big) +   \frac{C}{\sqrt{r}}\Big(1+ \sqrt{2 \log \frac{1}{\delta_2}}\Big),$$ with probability at least $1 - \delta_1 - \delta_2$.
\end{lemma}
\begin{proof}
Consider $\psi(x ; \omega_k) = \int_G \phi(gx,\omega_k) q(g) d\nu(g)$.  Let $\tilde{f}_k = \beta_k \ \psi(.;\omega_k), k =1 \cdots s$, with $\beta_k = \frac{\alpha(\omega_k)}{p(\omega_k)}$. Hence $ \mathbb{E}_{\omega_k\sim p} \tilde{f}_k = f^\star$.\\
Define $\tilde{f}(x) = \frac{1}{s}\sum_{k=1}^s \tilde{f}_{k}$. 
Let $\hat{f}_{k}(x)=\beta_{k} \hat{\psi}(x;\omega_{k})$, where $\hat{\psi}(x;\omega_{k}) = \frac{1}{r} \sum_{i=1}^r \phi(g_i x,\omega_k)$ is the empirical estimate of $\psi(x;\omega_k)$. 
Define $\hat{f}(x)= \frac{1}{s} \sum_{k=1}^s \hat{f}_k(x) $. We have $\mbb{E}_{g_i\sim q} \hat{f}(x) = \tilde{f}(x)$. 

$$\Big\| f^\star- \hat{f}\Big \|_{\mathcal{L}_2(X,\mu)} \leq \Big\| f^\star - \tilde{f}\Big \|_{\mathcal{L}_2(X,\mu)} + \Big\| \tilde{f} - \hat{f}\Big \|_{\mathcal{L}_2(X,\mu)}$$
From Lemma $1$ of \cite{rahimi2009weighted}, with probability $1-\delta_1$, $$\Big\| f^\star - \tilde{f}\Big \|_{\mathcal{L}_2(X,\mu)} \leq \frac{C}{\sqrt{s}}\Big( 1 + \sqrt{2\ \log \frac{1}{\delta_1}}\Big).$$ 

Since $\hat{f}(x) = \frac{1}{r}\sum_{i=1}^r \sum_{k=1}^s \frac{\beta_k}{s} \phi(g_ix,\omega_k)$ and $\mbb{E}_{g_i\sim q}\hat{f}(x) = \tilde{f}(x)$ with $g_i$ iid (and $\{\omega_k\}_{k=1}^s$ fixed beforehand), we can apply Lemma~\ref{lem:empavg} with
\begin{align*}
M = \left\lVert \sum_{k=1}^s \frac{\beta_k}{s} \phi(g_ix,\omega_k) \right\rVert \leq \sum_{k=1}^s \left\lvert \frac{\beta_k}{s} \right\rvert \lVert\phi(g_ix,\omega_k) \rVert  \leq \sum_{k=1}^s \left\lvert \frac{\beta_k}{s} \right\rvert \leq C. 
\end{align*}

\noindent We conclude that with a probability at least $1-\delta_2$,
$$ \Big\| \tilde{f} - \hat{f}\Big \|_{\mathcal{L}_2(X,\mu)} \leq \frac{C}{\sqrt{r}}\Big(1+ \sqrt{2 \log \frac{1}{\delta_2}}\Big).$$


\noindent Hence, with probability at least $1-\delta_1 - \delta_2$, we have 
$$\Big\| f^\star - \hat{f}\Big \|_{\mathcal{L}_2(X,\mu)}  \leq \frac{C}{\sqrt{s}}\Big( 1 + \sqrt{2\ \log \frac{1}{\delta_1}}\Big) +   \frac{C}{\sqrt{r}}\Big(1+ \sqrt{2 \log \frac{1}{\delta_2}}\Big)$$
\end{proof}

\begin{theorem}[Estimation error \cite{rahimi2009weighted}]
Let $\mathcal{F}$ be a bounded class of functions, $\sup_{x \in X} |f(x)| \leq C$ for all $f \in \mathcal{F}$. Let $V(y_i f(x_i))$ be an $L$-Lipschitz loss. Then with probability $1-\delta $, with respect to training samples $\{ x_i,y_i \}_{i = 1,2 \cdots N}$ (iid $\sim P$), every $f$ satisfies 
$$\mcal{E}_V(f) \leq \hat{\mcal{E}}_V(f) + 4 L \mathcal{R}_N(\mathcal{F}) + \frac{2|V(0)|}{\sqrt{N}}+ LC \sqrt{\frac{1}{2N}\log \frac{1}{\delta}},$$
where $\mcal{R}_N(\mcal{F})$ is the Rademacher complexity of the class $\mcal{F}$:
$$\mcal{R}_N(\mcal{F}) = \mbb{E}_{x,\sigma} \left[\sup_{f\in \mcal{F}} \left| \frac{1}{N} \sum_{i=1}^N \sigma_i f(x_i) \right| \right],$$
and $\sigma_i$ are iid symmetric Bernoulli random variables taking value in $\lbrace -1, 1\rbrace$, with equal probability and are independent form $x_i$. 
\end{theorem}
\begin{proof}
See in \cite{rahimi2009weighted}.
\end{proof}

Let $f \in \mathcal{F}_p$ and $\hat{f} \in \hat{\mathcal{F}}_p$ then the \emph{approximation error} is bounded as 
\begin{align*}
\mathcal{E}_{V}(\hat{f}) - \mathcal{E}_{V}(f) &\leq \mbb{E}_{(x,y)\sim P} \Big | V(y\hat{f}(x) ) -   V(y f(x) ) \Big |  \\
& \leq L \mbb{E} \big | \hat{f}(x) -  f(x) \big |   \\ 
& \leq L \sqrt{\mbb{E} \big ( \hat{f}(x) -  f(x) \big )^2 } \quad (\text{Jensen's inequaity for $\sqrt{\cdot}$ concave function})\\
& \leq LC \bigg ( \frac{1}{\sqrt{s}}\Big( 1 + \sqrt{2\ \log \frac{1}{\delta_1}}\Big) +   \frac{1}{\sqrt{r}}\Big(1+ \sqrt{2 \log \frac{1}{\delta_2}}\Big) \bigg ),
\end{align*}
with probability at least $1-\delta_1 -\delta_2$. 
Now let $f^\star_N = \argmin_{f \in \hat{\mathcal{F}}_p }\hat{\mcal{E}}_V(f)$ and $\tilde{f} = \argmin_{f \in \hat{\mathcal{F}}_p} \mcal{E}_V(f)$. We have
\begin{align*}
\mathcal{E}_{V}(f^\star_N ) - \min_{f \in \mathcal{F}_P} \mathcal{E}_{V}(f) &= \hat{\mathcal{E}_{V}}(f^\star_N ) - {\mathcal{E}_{V}}(\tilde{f} )  +{\mathcal{E}_{V}}(\tilde{f} ) - \min_{f \in \mathcal{F}_P} \mathcal{E}_{V}(f) \\
& \leq 2 \sup_{\tilde{f}\in \hat{\mathcal{F}}_P}  \Big | \mathcal{E}_{V}(\tilde{f})  - \hat{\mathcal{E}_{V}}(\tilde{f})\Big | + L \bigg ( \frac{C}{\sqrt{s}}\Big( 1 + \sqrt{2\ \log \frac{1}{\delta_1}}\Big) +   \frac{C}{\sqrt{r}}\Big(1+ \sqrt{2 \log \frac{1}{\delta_2}}\Big) \bigg ) \\
& \leq 2 \bigg( 4 L \mathcal{R}_N(\mathcal{F}) + \frac{2|V(0)|}{\sqrt{N}}+ LC \sqrt{\frac{1}{2N}\log \frac{1}{\delta}} \bigg) + \\& \qquad LC \bigg ( \frac{1}{\sqrt{s}}\Big( 1 + \sqrt{2\ \log \frac{1}{\delta_1}}\Big) +   \frac{1}{\sqrt{r}}\Big(1+ \sqrt{2 \log \frac{1}{\delta_2}}\Big) \bigg ),
\end{align*}
with probability at least $1-\delta - \delta_1 - \delta_2$. It is easy to show that $\mathcal{R}_N(\mathcal{F}) \leq \frac{C}{\sqrt{N}}$. Taking $\delta = \delta_1 = \delta_2$ yields the statement of the theorem.\\
\end{proof}

\end{document}